\documentclass{article} 
\usepackage{iclr2025_conference,times}
\usepackage{amsmath,amsfonts,bm}

\def\eqref#1{equation~\ref{#1}}

\def\1{\bm{1}}

\DeclareMathAlphabet{\mathsfit}{\encodingdefault}{\sfdefault}{m}{sl}
\SetMathAlphabet{\mathsfit}{bold}{\encodingdefault}{\sfdefault}{bx}{n}

\usepackage{hyperref}
\usepackage{booktabs}
\usepackage{xcolor}   
\usepackage{listings}

\usepackage{url}
\usepackage{graphicx} 
\usepackage{subcaption}
\usepackage{amsmath, amssymb, amsthm}
\newtheorem{proposition}{Proposition}

\title{Amortized Reasoning Tree Search: Decoupling Proposal and Decision in Large Language Models}

\author{
Zesheng Hong \\
Information Hub\\
The Hong Kong University of Science and Technology (Guangzhou)\\
No.1 Du Xue Rd, 511453, Guangzhou, China \\
\texttt{zhongal@connect.hkust-gz.edu.cn} \\
\And
Jiadong Yu \\
Information Hub\\
The Hong Kong University of Science and Technology (Guangzhou)\\
No.1 Du Xue Rd, 511453, Guangzhou, China \\
\texttt{jiadongyu@hkust-gz.edu.cn} \\
\AND
Hui Pan \\
Information Hub\\
The Hong Kong University of Science and Technology (Guangzhou)\\
No.1 Du Xue Rd, 511453, Guangzhou, China \\
\texttt{panhui@hkust-gz.edu.cn} \\
}

\iclrfinalcopy 
\begin{document}

\maketitle

\begin{abstract}
Reinforcement Learning with Verifiable Rewards (RLVR) has established itself as the dominant paradigm for instilling rigorous reasoning capabilities in Large Language Models. 
While effective at amplifying dominant behaviors, we identify a critical pathology in this alignment process: the systematic suppression of valid but rare (low-likelihood under base model distribution) reasoning paths.
We theoretically characterize this phenomenon as a \textbf{``Normalization Squeeze,''} where the interplay between mode-seeking policy gradients and finite sampling acts as a \textit{High-Pass Likelihood Filter}, driving the probability of rare correct traces to statistical extinction.
To counteract this collapse without discarding the base model's latent diversity, we propose \textbf{Amortized Reasoning Tree Search (ARTS)}. 
Unlike standard approaches that force \textit{internalization} via parameter updates, ARTS prioritizes \textit{deliberation} by decoupling generation from verification.
We introduce a \textbf{Flow Matching} objective that repurposes the verifier to estimate the conservation of probability flow, enabling robust navigation through sparse, high-entropy search spaces where traditional discriminative objectives fail.
Extensive experiments on the MATH-500 benchmark demonstrate that ARTS achieves a performance of 74.6\% (BoN@16), effectively matching fully fine-tuned policies (74.7\%) without modifying the generative backbone. 
Crucially, on the long-tail subset where coupled RL optimization collapses to 0\% pass@k, ARTS uniquely recovers significant performance, suggesting that disentangling verification from generation offers a more robust pathway for solving complex reasoning tasks.
\end{abstract}

\section{Introduction}
\label{sec:intro}

Reinforcement Learning with Verifiable Rewards (RLVR) has emerged as a cornerstone for advancing reasoning in Large Language Models (LLMs). Systems such as DeepSeek-R1~\cite{deepseek2025r1} and OpenAI o1~\cite{openai2024o1} demonstrate that optimizing for outcome correctness via Policy Gradient (e.g., GRPO, PPO) effectively internalizes reasoning capabilities, significantly boosting performance on standard benchmarks.
However, recent diagnostic studies, such as the work by~\cite{yue2025does}, reveal a concerning trend: while RLVR improves average performance by biasing the model towards dominant solutions, it often narrows the reasoning coverage, failing to expand—and sometimes even reducing—the model's intrinsic capacity to solve diverse problems.

In this work, we aim to uncover the mechanism behind this phenomenon and explore effective countermeasures. 
We argue that the prevailing "internalization" paradigm acts as a \textbf{High-Pass Likelihood Filter}, systematically extinguishing valid but rare reasoning patterns that reside in the long tail of the base model's distribution. As illustrated in Figure ~\ref{fig:normalization_squeeze_a}, this optimization dynamic creates an \textbf{"Extinction Zone"}, where the probability mass of valid sparse solutions is suppressed to near zero by the mode-seeking policy.

To address this, we revisit the classical \textbf{"Generate-then-Verify"} paradigm~\cite{cobbe2021training, lightman2023letsverify}. 
While separating verification from generation is not new, previous approaches typically rely on \textbf{discriminative objectives}—either pointwise regression (e.g., MSE in PRMs) or pairwise ranking (e.g., Bradley-Terry in RMs). We identify that both paradigms struggle in sparse signal regimes, as they inherently penalize exploration paths that do not yield immediate high confidence or form valid comparison pairs.
We propose to view verification through the lens of \textbf{Flow Matching}~\cite{bengio2021flow}, investigating whether a flow-calibrated objective can better preserve the sensitivity to rare correct paths that standard discriminative objectives ignore.

Our contributions are threefold, spanning theoretical analysis, methodological investigation, and empirical validation:

\begin{itemize}
    \item \textbf{Theoretical Mechanism of Extinction:} We provide a formal analysis of the optimization dynamics in RLVR (Section~\ref{sec:theoretical_analysis}). We derive how the interplay between low-temperature sampling and finite rollouts induces a Normalization Squeeze, where the partition function shift causes the probability of unsampled valid traces to decay exponentially. This offers a theoretical explanation for the "coverage narrowing" phenomenon observed in recent empirical studies~\cite{yue2025does}.
    
    \item \textbf{Revisiting Verification with Flow Matching:} We introduce \textbf{Amortized Reasoning Tree Search (ARTS)}, a framework that repurposes the verifier as a flow estimator. Unlike standard PRMs that regress to local step quality, ARTS leverages the conservation of probability flow to capture the aggregated value of future trajectories. This allows us to study whether modeling global uncertainty offers superior robustness in high-entropy search spaces compared to traditional discriminative baselines.
    
    \item \textbf{Empirical Analysis on the Long Tail:} We conduct controlled experiments on MATH-500 to evaluate the trade-offs between training and search. 
Crucially, on the "Hardest Subset" where the RL-tuned policy fails completely (0\% pass@16), we observe that standard verifiers (both Pointwise PRM and Pairwise RM) also struggle, yielding results often \textit{worse than random sampling} (1.7\% vs 2.7\%).
In contrast, ARTS is the \textbf{only method to maintain a positive gain} (6.9\%, closing 25.8\% of the gap), confirming its unique structural advantage in preserving valid reasoning signals that are statistically invisible to other paradigms.
\end{itemize}

This work does not claim to solve the reasoning challenge; rather, it offers a rigorous study of the \textit{data dynamics} in alignment. 
Our findings suggest that for hard, long-tail problems, deferring computation to inference-time search with a flow-based verifier offers a more robust alternative to forcing internalization via reinforcement learning.

\section{Related Work}
\label{sec:related_work}

\subsection{Reinforcement Learning with Verifiable Rewards (RLVR)}
RLVR has established itself as the standard for aligning LLMs in domains with ground-truth feedback. 
Iterative self-improvement methods like \textbf{STaR}~\cite{zelikman2022star} and \textbf{ReST}~\cite{gulcehre2023reinforced} refine reasoning traces through repeated sampling and filtering. 
More recently, state-of-the-art systems such as \textbf{DeepSeek-R1}~\cite{deepseek2025r1} and \textbf{DeepSeekMath}~\cite{shao2024deepseekmath} utilize \textbf{Group Relative Policy Optimization (GRPO)} to stabilize training without value networks. While preference-optimization methods like \textbf{DPO}\cite{rafailov2023direct} and \textbf{KTO}\cite{ethayarajh2024kto} demonstrate strong performance in general alignment, RLVR remains the preferred approach for reasoning tasks requiring verifiable correctness. Despite their success in \textit{internalizing} capabilities, recent diagnostic studies by ~\cite{yue2025does} reveal a critical side effect: RL training acts as a mode-seeking process that often reduces reasoning coverage on long-tail problems. 
Our work builds on this diagnosis by identifying the theoretical mechanism—the "Normalization Squeeze"—and proposing a parameter-efficient search framework that avoids this coverage collapse.

\subsection{Inference-Time Search and Verification}
Structuring reasoning as a search process ("System 2") is a promising direction for handling complex tasks, exemplified by \textbf{Chain-of-Thought}~\cite{wei2022chain} and \textbf{Tree of Thoughts}~\cite{yao2024tree}. This paradigm has evolved significantly: from simple majority voting in \textbf{Self-Consistency}\cite{wang2023selfconsistency}, to assessing intermediate steps in \textbf{WizardMath}
\cite{luo2024wizardmath}, and exploring complex graph topologies in \textbf{RAP}\cite{hao2023reasoning}, \textbf{LATS}\cite{zhou2024language}, and \textbf{Graph of Thoughts}~\cite{besta2024graph}.
The efficacy of these methods hinges on the quality of the verifier. 
Standard approaches train \textbf{Process Reward Models (PRMs)} using pointwise regression (e.g., \textit{Math Shepherd}~\cite{lightman2023letsverify}) or pairwise ranking (e.g., \cite{cobbe2021training}).
However, we find that these \textbf{discriminative objectives} struggle in sparse signal regimes (the "needle in a haystack" scenario), as they inherently penalize exploration paths that do not yield immediate high confidence. 
In contrast, \textbf{ARTS} repurposes the verifier as a \textbf{flow estimator}, leveraging the flow conservation property to aggregate global probability mass, thereby providing a more robust signal for navigating high-entropy search spaces.

\subsection{Generative Flow Networks (GFlowNets)}
\textbf{Generative Flow Networks (GFlowNets)}~\cite{bengio2021flow} provide a probabilistic framework for sampling diverse candidates proportional to a reward function. Advanced training objectives, such as \textbf{Trajectory Balance}~\cite{malkin2022trajectory}, have further stabilized the credit assignment in these generative flow networks.
Recent works have applied GFlowNet principles to language modeling, primarily to fine-tune the generative policy for diverse sequence generation~\cite{yu2024flow,zhu2025flowrl}.
Distinct from these approaches that focus on reshaping the \textit{generator}, we employ the flow matching objective solely to train a \textbf{verifier} over a frozen, high-entropy proposer. 
This decoupled design is specifically tailored to preserve the base model's latent diversity ("Resurrecting Extinguished Knowledge") rather than forcing the policy to collapse onto a single mode.

\section{Theoretical Analysis: The Dynamics of Reasoning Extinction}
\label{sec:theoretical_analysis}

In this section, we analyze the optimization dynamics of Reinforcement Learning with Verifiable Rewards (RLVR). 
While \textbf{Policy Gradient Theorem} \cite{sutton1999policy} guarantees convergence to optimal policies under infinite sampling, we demonstrate that the practical constraints of \emph{low-temperature sampling} and \emph{finite rollouts} introduce a systemic selection bias.
Our analysis reveals a critical failure mode: valid reasoning patterns that fall below a \textbf{Visibility Barrier} are subject to exponential extinction, regardless of their correctness.

\begin{figure}[t]
    \centering
    
    \begin{subfigure}[t]{0.48\linewidth}
        \centering
        \includegraphics[width=\linewidth]{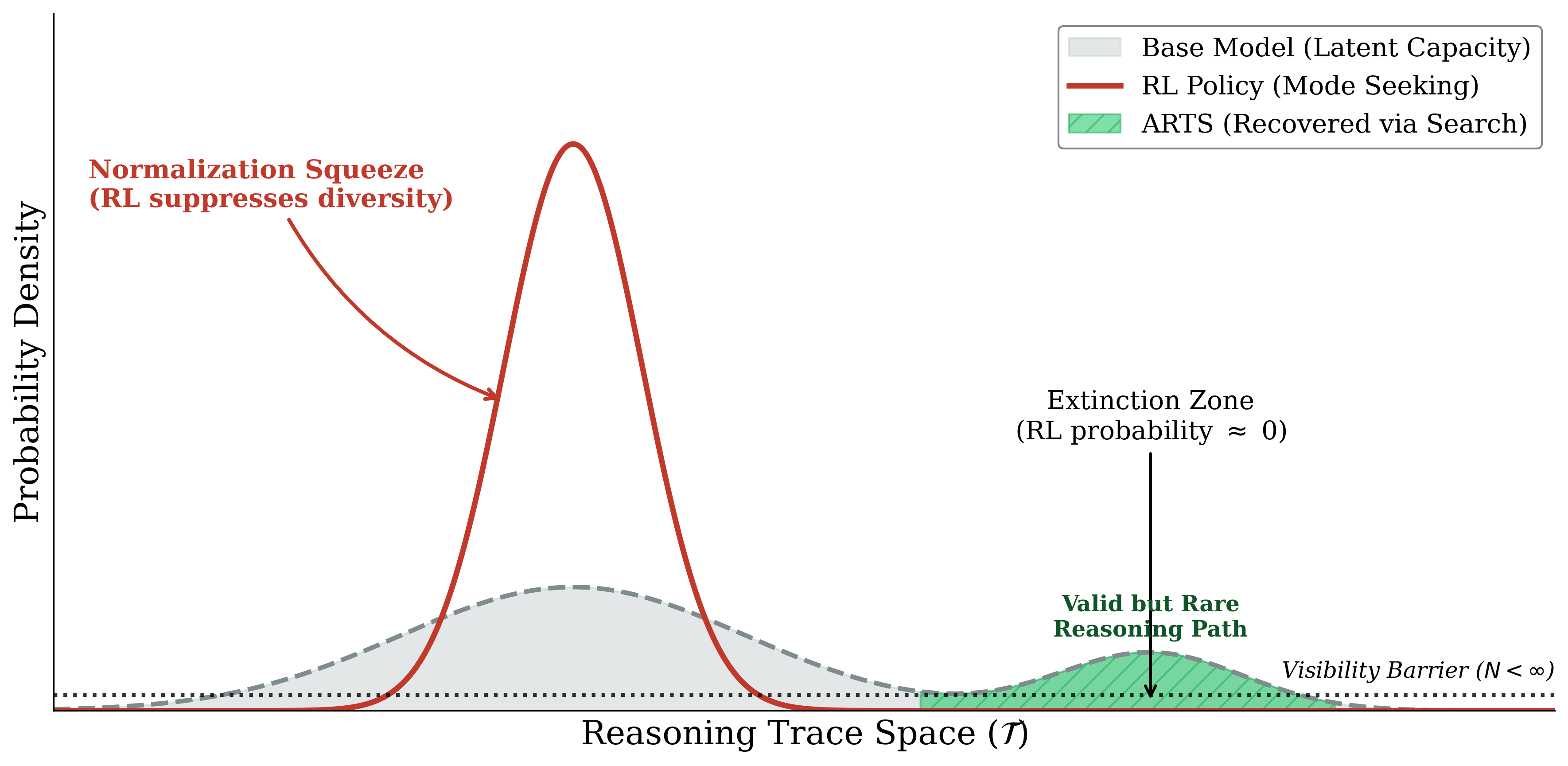}
        \caption{}
        \label{fig:normalization_squeeze_a}
    \end{subfigure}
    \hfill
    \begin{subfigure}[t]{0.48\linewidth}
        \centering
        \includegraphics[width=\linewidth]{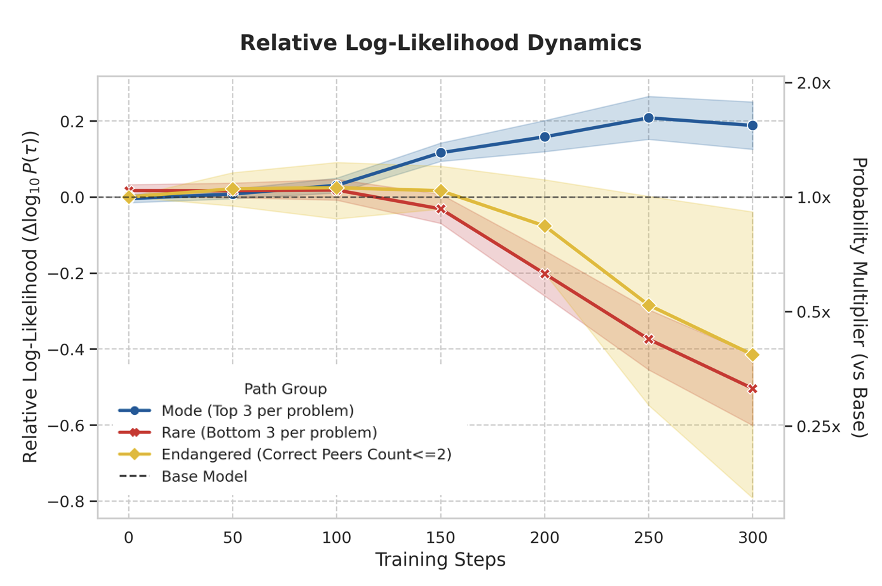}
        \caption{}
        \label{fig:normalization_squeeze_b}
    \end{subfigure}

    \caption{
\textbf{(a) Visualizing the ``Normalization Squeeze''.} 
Base Model (Gray): Covers both common and rare valid reasoning paths. 
RL Policy (Red): Acts as a \textit{High-Pass Filter}, amplifying the dominant mode while suppressing rare traces to near zero (\textit{Extinction Zone}). 
\textbf{(b) Visualization of Capability Extinction.} 
We track the relative log-likelihood evolution ($\Delta \log_{10} P(\tau)$) of valid reasoning traces over 300 GRPO steps.
While dominant Mode Traces (Blue) are consistently reinforced, Rare (Red) and Endangered (Yellow) traces suffer systematic decay despite being fully correct.
This confirms that standard RLVR acts as a High-Pass Likelihood Filter, preserving only the high-redundancy patterns while purging the model's capability to solve hard, long-tail problems.
ARTS (Green): Decouples generation from verification, using inference-time search to recover these extinguished paths and unlock latent long-tail capacity.
}
    \label{fig:normalization_squeeze}
\end{figure}

Let $\mathcal{T}$ denote the discrete space of reasoning traces. 
Standard RLVR algorithms (e.g., GRPO~\cite{shao2024deepseekmath}, PPO~\cite{schulman2017proximal}) aim to maximize a reward objective subject to a KL constraint against a reference policy $\pi_{\text{ref}}$. 
The theoretical update at iteration $k$ is often formulated as:
\begin{equation}
    \theta_{k+1} = \mathop{\arg\max}_{\theta} \mathbb{E}_{\tau \sim \pi_{\theta_k}} \left[ R(\tau) \right] - \beta \mathrm{KL}(\pi_\theta \| \pi_{\text{ref}}).
\end{equation}
However, practical implementations approximate the expectation using a finite set of samples $\mathcal{S}_k = \{\tau_1, \dots, \tau_N\}$ drawn from a proposal distribution $Q_k$. 
This induces a \textbf{Sampling-Bias Gap}: the gradient update is computed exclusively on the \emph{sparse empirical support} $\mathcal{S}_k$, creating a disconnect between the theoretical objective and the realized optimization trajectory.

\paragraph{Assumption (Function Space Dynamics).}
To isolate the effects of selection bias from generalization noise, we analyze the dynamics in the function space of reasoning traces. We assume that distinct reasoning paths are approximately orthogonal in the representation space, such that updating the logits of sampled traces does not arbitrarily boost unsampled traces via parameter sharing.

\paragraph{Practical Constraint: Low-Temperature Filtering.}
While unbiased exploration ($T=1$) is theoretically ideal, state-of-the-art reasoning methods typically operate in a \emph{low-temperature regime} to ensure logical coherence and reduce hallucinations \cite{zelikman2022star, cobbe2021training}.
Leading reasoning frameworks, including DeepSeek-R1 \cite{deepseek2025r1} and ReST \cite{gulcehre2023reinforced}, explicitly utilize sharpened sampling or rejection mechanisms to filter out low-quality traces:
\begin{equation}
\label{eq:sampling_dist}
    Q_k(\tau) \propto \pi_k(\tau)^{1/T}, \quad \text{with } T \in [0.6, 0.8].
\end{equation}
Even when $T=1$, the constraint of finite rollouts ($N \ll |\mathcal{T}|$) imposes an implicit \textbf{High-Pass Likelihood Filter}. 
A reasoning trace with $\pi(\tau) \ll 1/N$ is statistically invisible to the optimization process, effectively creating a truncation equivalent to $T_{\text{eff}} < 1$.

\subsection{Mechanism: The Normalization Squeeze}

We now derive the closed-form evolution of reasoning probabilities under this sparse update regime.
We model the update dynamics as a \emph{Partial Boltzmann Evolution}, where the probability mass redistribution is constrained by the limited support of sampled traces.

\begin{proposition}[Probability Decay under Sparse Updates]
\label{prop:decay_law}

Let $\mathcal{S}_k$ denote the set of sampled traces at iteration $k$. 
For any valid reasoning trace $\tau$ that is not sampled 
(i.e., $\tau \notin \mathcal{S}_k$), its probability under the updated policy 
$\pi_{k+1}$ satisfies:

\begin{equation}
\label{eq:closed_form_decay}
    \pi_{k+1}(\tau) 
    = \pi_k(\tau) 
    \left( \frac{1}{1 + \alpha_k} \right),
\end{equation}

where $\alpha_k > 0$ represents the partition function shift induced by the rewarded set:

\begin{equation}
    \alpha_k 
    = \sum_{\tau' \in \mathcal{S}_k} 
    \pi_k(\tau') 
    \left[
        \exp\!\left( \frac{R(\tau')}{\beta} \right) - 1
    \right].
\end{equation}

\end{proposition}

\begin{proof}
See Appendix~\ref{app:proof_thm1} for the full derivation based on partition function shifts.
\end{proof}

\paragraph{Correctness $\neq$ Survival.}
Eq.~\eqref{eq:closed_form_decay} establishes a fundamental insight often overlooked in RLVR: \textbf{validity is not a sufficient condition for reinforcement}.
The survival of a reasoning trace depends on two factors:
\begin{enumerate}
    \item \textbf{Visibility:} Whether it can cross the sampling threshold to enter $\mathcal{S}_k$.
    \item \textbf{Competition:} Whether it can withstand the ``Normalization Squeeze'' ($\gamma_k$)—a passive suppression force exerted by the strictly increasing partition function of the dominant modes.
\end{enumerate}
For rare correct traces, the feedback loop is strictly contractive: they are missed by sampling ($T<1$) $\to$ their logits are frozen $\to$ their probability decays due to normalization $\to$ they become even less likely to be sampled. This leads to the irreversible extinction of long-tail reasoning capabilities.

\subsection{Empirical Validation: The Extinction of Hard Capabilities}
\label{sec:empirical_validation}

To rigorously test our theoretical predictions, we conducted a large-scale controlled experiment tracking the lifecycle of reasoning traces during standard GRPO training.
Unlike previous studies that focus only on general reward curves, we specifically isolate the fate of ``Endangered'' reasoning patterns—valid solutions to hard problems that are statistically fragile.

\paragraph{Experimental Setup.}
We utilize \textbf{Qwen-2.5-7B}~\cite{qwen2024technical} as the base policy and evaluate on the \textbf{MATH-500} benchmark~\cite{lightman2023letsverify}, a subset of the MATH dataset specifically curated to test rigorous reasoning.
To ensure coverage of the natural reasoning distribution, we first construct a fixed offline pool by sampling $N=64$ responses per problem from the base model at temperature $T=1.0$.
We then perform 300 steps of standard GRPO training using this fixed pool.
Crucially, we categorize all \emph{correct} reasoning traces into three distinct groups based on their redundancy and initial likelihood:

\begin{enumerate}
    \item \textbf{Mode Traces (Blue):} The top-2 highest likelihood correct traces for each problem. These represent the model's dominant reasoning habits.
    \item \textbf{Rare Traces (Red):} The bottom-2 lowest likelihood correct traces for each problem. These represent valid but "long-tail" reasoning strategies.
    \item \textbf{Endangered Traces (Yellow):} Traces from problems that have \textbf{$\le 2$ correct paths total} in the pool. These represent the \textbf{sole surviving valid reasoning paths} for the hardest problems, serving as a proxy for the model's \emph{fragile knowledge boundary}.
\end{enumerate}

\paragraph{Metric Definition.}
Throughout training, we track the \textbf{relative log-likelihood} assigned to each reasoning trace $\tau$ by the evolving policy $\pi_t$, measured relative to the base model $\pi_0$:
\begin{equation}
    \Delta \log_{10} \mathcal{L}(\tau)
    =
    \log_{10} \pi_t(\tau)
    -
    \log_{10} \pi_0(\tau)
    =
    \log_{10}\!\left(
    \frac{\pi_t(\tau)}{\pi_0(\tau)}
    \right).
\end{equation}
A value of $+1.0$ indicates a $10\times$ increase in probability mass, while $-1.0$ indicates a $10\times$ suppression. This metric isolates the optimization dynamics from the intrinsic complexity of the trace.

\paragraph{Results: The Collapse of Fragile Knowledge.}
Figure~\ref{fig:normalization_squeeze_b} visualizes these dynamics over 300 training steps, revealing a distinct phase transition:

\begin{itemize}
    \item \textbf{The Rich-Get-Richer Effect (Mode):} 
    Dominant traces (Blue), defined as the top-2 highest likelihood paths, consistently gain probability mass. They reach a relative log-likelihood of $\Delta \approx +0.2$ (representing a $1.6\times$ probability boost), confirming that RLVR efficiently reinforces frequent patterns.
    
    \item \textbf{The Extinction of Diversity (Rare):} 
    In contrast, Rare traces (Red)—the bottom-2 lowest likelihood paths—suffer a catastrophic decline. Despite being fully correct, their log-likelihood drops by $\Delta \approx -0.5$ (shrinking to $\sim 0.3\times$ original probability). This validates Proposition~\ref{prop:decay_law}: without sufficient visibility, validity alone cannot protect a trace from the normalization squeeze.

    \item \textbf{Capability Collapse (Endangered):} 
    Most critically, the Endangered traces (Yellow)—representing problems with sparse solutions ($\le 2$ total found)—mirror the extinction trajectory of the Rare group, crashing to $-0.4$. 
    Since these traces constitute the \emph{sole} valid reasoning capability for these hard problems, their suppression implies a net \textbf{loss of problem-solving capability}. 
    The model effectively "forgets" how to solve hard problems because their sparse solutions fail to survive the \textbf{High-Pass Likelihood Filter} imposed by finite sampling.
\end{itemize}

This experiment provides empirical proof that standard RLVR training acts as a \textbf{High-Pass Filter}, systematically purging reasoning capabilities that lack initial redundancy.

\section{Methodology}
\label{sec:method}

To counteract the boundary contraction issue identified in Section~\ref{sec:theoretical_analysis}, we propose \textbf{Amortized Reasoning Tree Search (ARTS)}. 
Unlike coupled RLVR methods that collapse the search space to the base model's priors, ARTS structurally decouples the \textit{Proposer} (System 1) from the \textit{Verifier} (System 2)~\cite{kahneman2011thinking, yao2024tree}. This allows the Verifier to learn a global value function that transcends the myopic biases of the generative policy. This section details the ARTS framework, the \textbf{Sparse Deep Fork} data construction, and formally derives our training objective as a \textbf{Flow Matching Lower Bound} grounded in Generative Flow Network (GFlowNet) theory.

\subsection{Problem Setup}
\label{sec:problem_setup}

We formalize multi-step reasoning as a sequential decision process on a Directed Acyclic Graph (DAG) $\mathcal{G}$. 
A state $s_t = (x, a_1, \dots, a_t)$ represents a partial reasoning chain.
The environment provides feedback only at terminal states $s_T$. 
To mitigate the \emph{vanishing gradient problem} inherent in sparse reward environments and ensure \emph{exploration continuity}, we adopt a \textbf{Log-Flow Smoothing} reward scheme:
\begin{equation}
\label{eq:soft_reward}
R(s_T) = 
\begin{cases}
    10, & \text{if } s_T \text{ is correct}, \\
    0.1, & \text{otherwise}.
\end{cases}
\end{equation}
This contrastive design assigns high flow magnitude to correct reasoning ($100\times$) while maintaining a non-zero \textbf{background flow} for incorrect paths, preventing the flow manifold from collapsing to singularity ($\log 0 \to -\infty$).

\subsection{Amortized Reasoning Tree Search (ARTS)}
\label{sec:ARTS}

The goal of ARTS is to train a Verifier $V_\theta(s)$ that guides search at inference time to recover correct paths invisible to standard sampling.
We utilize the frozen base model $\pi_0$ as a \textit{Proposer}. 
While $\pi_0$ may assign low probability to a correct trajectory $\tau^*$, the probability of discovering $\tau^*$ via a stochastic tree search with width $K$ expands geometrically as $K^H \cdot \pi_0(\tau^*)$. 
We term this phenomenon \textbf{Search-Based Probability Amplification}: theoretically, the probability of discovering a rare trajectory $\tau^*$ expands geometrically as $1 - (1 - \pi_0(\tau^*))^K \approx K \cdot \pi_0(\tau^*)$ in the small probability limit.

\begin{figure}[t]
    \centering
    \includegraphics[width=\linewidth]{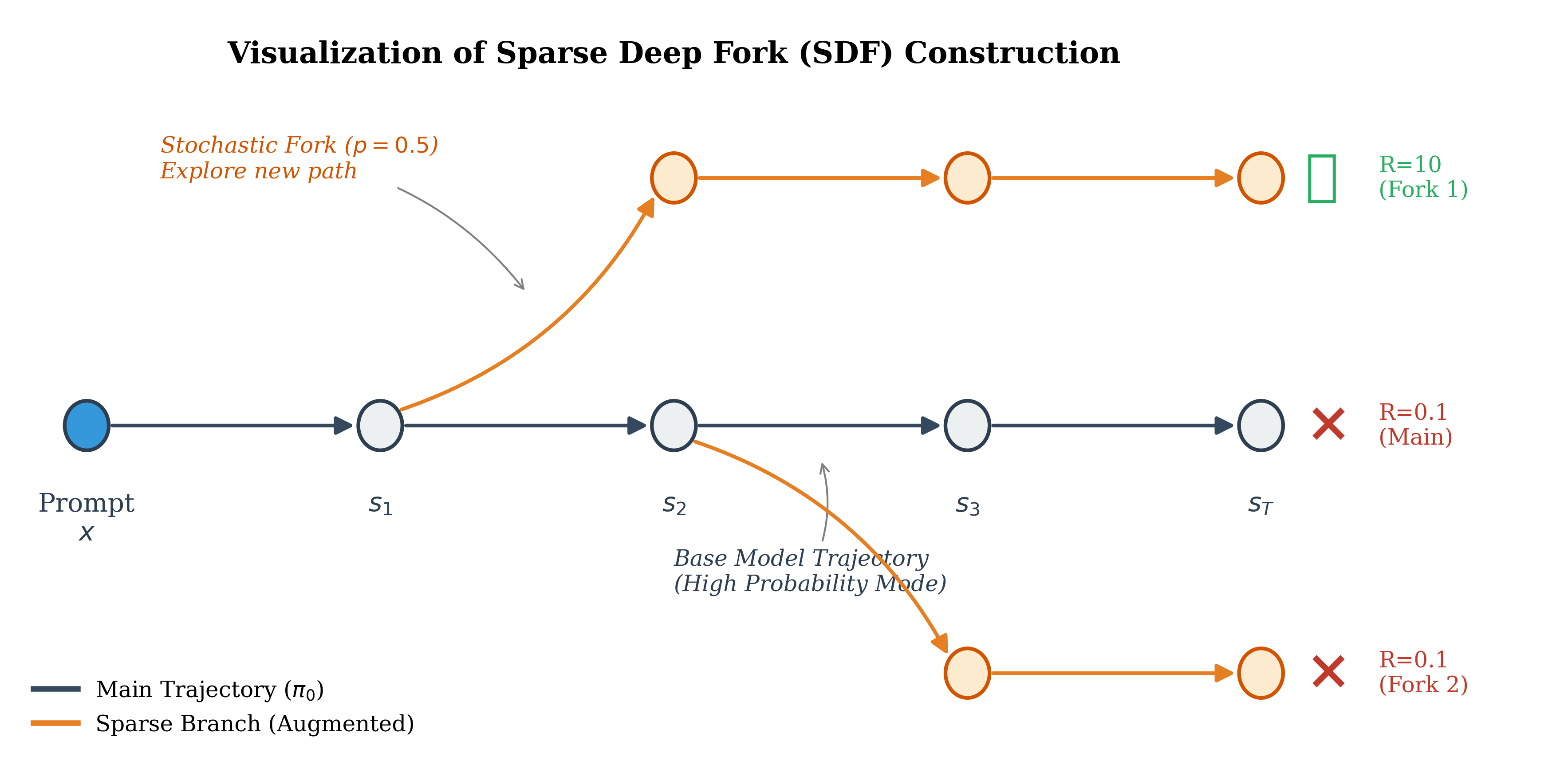}
    \caption{
    \textbf{Visualizing Sparse Deep Fork (SDF).} 
    Instead of exhaustive tree search, SDF constructs a cost-effective DAG by stochastically branching (Orange) from the base model's main trajectory (Blue). 
    This structure efficiently exposes the verifier to local decision boundaries and alternative outcomes—such as discovering a correct path ($R=10$) diverging from an incorrect one—without the exponential cost of full expansion.
}
\label{fig:sdf_diagram}
\end{figure}

\subsection{Data Construction: Sparse Deep Fork}
\label{sec:data_construction}

Training a flow function requires observing the branching topology of the reasoning space. However, constructing a full tree is computationally prohibitive.
We propose \textbf{Sparse Deep Fork (SDF)}, a topological approximation strategy illustrated in Figure~\ref{fig:sdf_diagram}.

The graph construction proceeds as follows:
\begin{enumerate}
    \item \textbf{Main Trajectory (Blue):} For a given problem $x$, we sample $N$ independent reasoning trajectories from the base policy $\pi_0$. These serve as the "backbone" of our topological map.
    
    \item \textbf{Stochastic Branching (Orange):} We iterate through intermediate nodes $s_{t}$ in the main trajectory. With probability $p=0.5$, we select $s_{t}$ as a pivot point and perform \textbf{one additional rollout} (fork) using $\pi_0$. 
\end{enumerate}

\textbf{Efficiency Analysis.} This generates a sparse DAG $\mathcal{G}_{\text{sdf}}$ where specific nodes have two observed children (the original path and the augmented branch). 
As shown in Figure~\ref{fig:sdf_diagram}, this structure allows the verifier to observe \textbf{counterfactual outcomes}—e.g., a correct branch ($R=10$) diverging from an incorrect main path ($R=0.1$)—using only $\mathcal{O}(1.5N)$ compute. 
This efficiently exposes the local decision boundaries necessary for training a discriminative flow function.

\subsection{Reasoning Flow Matching (RFM)}
\label{sec:RFM}

Our training objective is grounded in the theory of \textbf{Generative Flow Networks (GFlowNets)} \cite{bengio2021flow}, which learn a flow function $F(s)$ such that the flow into any state equals the flow out of it.
However, standard GFlowNet training assumes access to either the full graph or an on-policy exploration policy. In our sparse, offline setting, we adapt the objective to learn from the \textbf{observed support}.

\paragraph{The Partial Consistency Objective.}
We parameterize the verifier $F_{\theta}(s): \mathcal{S} \to \mathbb{R}^+$ in log-space. 
For any non-terminal state $s_t \in \mathcal{G}_{\text{sdf}}$ with the set of \emph{observed} children $\mathcal{C}_{obs}(s_t)$, we minimize the local flow mismatch:
\begin{equation}
\label{eq:fm_consistency}
\mathcal{L}_{\mathrm{RFM}}(\theta) = \mathbb{E}_{s_t \sim \mathcal{G}_{\text{sdf}}} \left[ \left( \log F_{\theta}(s_t) - \log \sum_{s' \in \mathcal{C}_{obs}(s_t)} F_{\theta}(s') \right)^2 \right] + \lambda \mathcal{L}_{\text{Leaf}},
\end{equation}
where $\mathcal{L}_{\text{Leaf}}$ regresses the terminal flow to the reward $R(s_T)$.

\paragraph{Analysis: Learning a Ranking-Preserving Lower Bound.}
A critical theoretical question arises: \emph{Does minimizing Eq.~\ref{eq:fm_consistency} on a sparse graph recover the true flow values?}
Formally, the true flow consistency equation requires summing over \emph{all} possible children $\mathcal{C}_{all}(s_t)$:
\begin{equation}
    F^*(s_t) = \sum_{s' \in \mathcal{C}_{all}(s_t)} F^*(s').
\end{equation}
Since our observed set is a subset $\mathcal{C}_{obs} \subset \mathcal{C}_{all}$, the training target is inherently an underestimation:
\begin{equation}
    \text{Target}(s_t) = \sum_{s' \in \mathcal{C}_{obs}} F(s') \leq \sum_{s' \in \mathcal{C}_{all}} F(s') = F^*(s_t).
\end{equation}
Consequently, minimizing the local mismatch on $\mathcal{C}_{obs}$ enforces $F_\theta$ to converge to a \textbf{Conservative Flow Estimate}:
\begin{equation}
    F_\theta(s_t) \to \sum_{s' \in \mathcal{C}_{obs}} F(s') \le \sum_{s' \in \mathcal{C}_{all}} F^*(s').
\end{equation}
This effectively creates a \textbf{Ranking-Preserving Lower Bound} where the flow value represents the \emph{confirmed} probability mass of success discovered so far.

\paragraph{Sufficiency for ARTS.}
While $F_\theta$ underestimates the exact probability, we argue that this lower bound is \textbf{sufficient and effective} for the specific task of ranking reasoning paths in ARTS for three reasons:

\begin{enumerate}
    \item \textbf{Dominance of Correct Flows:} Due to our reward design ($10$ vs $0.1$), the magnitude of flow is dominated by the presence of correct paths. Even if we miss $90\%$ of the children, observing a \emph{single} correct child contributes $+10$ to the sum, whereas observing ten incorrect children contributes only $+1$. Thus, the lower bound $F_\theta(s)$ remains high if and only if at least one successful path is discovered.
    
    \item \textbf{Monotonicity Preservation:} The flow matching objective enforces $F(s_t) \approx \sum F(s_{t+1})$. Since flows are non-negative, this implies $F(s_t) \ge F(s_{t+1})$ (for single-child chains). The learned function preserves the monotonicity of the value landscape, ensuring that the gradient points towards the sources of reward.
    
    \item \textbf{Ranking Consistency:} For inference, we only require the correct ranking $F(s_{\text{good}}) > F(s_{\text{bad}})$. 
    Consider a state $s_{\text{good}}$ that leads to a correct answer. Even with sparse sampling, there is a probability that the correct path is included in $\mathcal{C}_{obs}$, pushing $F(s_{\text{good}}) \to 10$.
    Conversely, for a dead-end state $s_{\text{bad}}$, \emph{all} possible children lead to $0.1$. No matter how many children we miss, the sum cannot exceed the background noise level.
    Therefore, the \textbf{contrastive margin} between valid and invalid states is topologically robust: sparse sampling may underestimate the \emph{magnitude} of a correct node, but it strictly bounds the invalid nodes to the noise floor, preserving the necessary gradient for search guidance.
\end{enumerate}

\section{Experiments}
\label{sec:experiments}

We evaluate ARTS on the \textbf{MATH-500} benchmark,\footnote{All data utilized for experiments are publicly available} comparing our decoupled search framework against standard baselines in a controlled experimental setting. We focus our analysis on the MATH benchmark~\cite{hendrycks2021measuring} as it serves as a standard testbed for measuring multi-step reasoning capabilities.
\begin{table*}[htbp]
\centering
\caption{\textbf{Main Results on MATH-500: Efficiency and Scaling.} 
\textbf{The "Free Lunch" of Inference Search:} 
While GRPO achieves 74.7\% via expensive fine-tuning, ARTS matches this performance (\textbf{74.6\%}) using a frozen base model and a lightweight verifier.
\textbf{Scaling Dynamics:} 
We observe a distinct \textit{crossover}: Pointwise PRM dominates at low budgets ($N \le 8$), acting as a "precision filter." 
However, ARTS exhibits stronger scaling laws, overtaking PRM at $N=16$. 
This confirms that Flow Matching preserves the \textit{diversity} required to recall valid solutions from the long tail, whereas discriminative verifiers saturate early.}
\label{tab:main_results}
\resizebox{\textwidth}{!}{%
\begin{tabular}{l|c|c|c|cccc|c}
\toprule
& & \textbf{Training} & \textbf{Generation} & \multicolumn{4}{c|}{\textbf{Inference-Time Search (Best-of-N)}} & \\
\textbf{Method} & \textbf{Backbone} & \textbf{Params} & \textbf{Pass@1} & \textbf{BoN@2} & \textbf{BoN@4} & \textbf{BoN@8} & \textbf{BoN@16} & \textbf{Gap Closed} \\
\midrule
Qwen-2.5-7B (Base) & 7B & - & 57.7\% & 63.6\% & 68.2\% & 71.5\% & 73.9\% & - \\
SimpleRL-Zoo (GRPO) & 7B & 100\% & \textbf{74.7\%} & - & - & - & - & - \\
\midrule
Base + Pairwise RM & 1.5B & $<$0.3\% & 57.7\% & 63.6\% & 66.2\% & 67.0\% & 67.4\% & 31.3\% \\
Base + Pointwise PRM & 1.5B & $<$0.3\% & 57.7\% & \textbf{67.1\%} & \textbf{71.6\%} & \textbf{73.5\%} & 73.8\% & 52.1\% \\
\textbf{Base + ARTS (Ours)} & 1.5B & $<$0.3\% & 57.7\% & 66.0\% & 70.1\% & 72.2\% & \textbf{74.6\%} & \textbf{54.7\%} \\
\bottomrule
\end{tabular}%
}
\begin{flushleft}
\small
\textbf{Note:} Base Model BoN scores are calculated using random sampling (Oracle Upper Bound for search). \textbf{Gap Closed} is based on BoN@16.
\end{flushleft}
\end{table*}
\subsection{Experimental Setup}

\paragraph{Models and Architecture.}
\begin{itemize}
    \item \textbf{Proposer (Base):} We employ the frozen \textit{Qwen-2.5-7B} \cite{qwen2024technical} as the reasoning generator. We deliberately use the base model to preserve maximum entropy and search coverage.
    
    \item \textbf{Verifier Backbone:} We utilize the lightweight \textit{Qwen-2.5-1.5B-Instruct} as a Causal Feature Extractor, operating in \texttt{bfloat16} precision with Flash Attention 2.
    
    \item \textbf{LoRA Configuration:} To ensure parameter efficiency, we apply Low-Rank Adaptation (LoRA) \cite{hu2021lora} with $r=16, \alpha=32$, and dropout $0.1$. Adaptation targets all linear projections (\texttt{q}, \texttt{k}, \texttt{v}, \texttt{o}, \texttt{gate}, \texttt{up}, \texttt{down}), resulting in updating only $\sim$0.3\% of total parameters.
    
    \item \textbf{Value Head \& Aggregation:} Instead of a linear probe, we use a \textbf{2-layer MLP} (Hidden $\to$ GELU $\to$ Dropout $\to$ Scalar). Crucially, we employ \textbf{Last-Token Aggregation}: the flow value $F(s_t)$ is predicted from the hidden state of the last token in step $s_t$, enforcing strict causality.
\end{itemize}

\paragraph{Training Data: Sparse Deep Fork.}
All verifiers are trained on the same dataset seeded from the \textbf{MATH training split} \cite{hendrycks2021measuring}. 
For each problem, we sample $N=16$ trajectories via the Proposer ($T=0.6, p=0.95$) and apply stochastic branching ($p=0.5$) to generate local forks. 
Terminal states are verified against ground truth solutions and labeled with our soft reward scheme ($R \in \{0.1, 10\}$).

\paragraph{Baselines.}
We compare ARTS against representative methods:
\begin{enumerate}
    \item \textbf{SimpleRL-Zoo (GRPO)} \cite{zeng2025simplerl}: 
    An RLVR baseline initialized from \texttt{Qwen-2.5-7B} and fully fine-tuned using Group Relative Policy Optimization on MATH and GSM8K.
    
    \item \textbf{Pointwise PRM} \cite{lightman2023letsverify}: 
    A verifier sharing the ARTS architecture (1.5B + LoRA), trained with \textbf{Mean Squared Error (MSE)} regression on step-level rewards (Math-Shepherd style).
    
    \item \textbf{Pairwise RM} \cite{ouyang2022training}: 
    A verifier sharing the ARTS architecture, trained with the \textbf{Bradley-Terry ranking loss} on the branching nodes of the Sparse Deep Fork graph.
\end{enumerate}

\paragraph{Implementation Details.}
We perform \textbf{Best-of-$N$} sampling ($N \in \{2, 4, 8, 16\}$) for all verifier-guided methods.
We generate $N$ independent chains using the frozen Proposer. 
Trajectory scoring differs by objective: for Pointwise PRM and Pairwise RM, we aggregate step scores using the \textbf{minimum} value (worst-case verification); for ARTS, we use the \textbf{terminal} flow value $F(s_T)$ consistent with flow conservation.

\subsection{Efficiency and Effectiveness of Decoupled Search}

Table~\ref{tab:main_results} presents the comparative analysis between coupled policy optimization (GRPO) and decoupled search (ARTS) on the MATH-500 benchmark. 
The results reveal a distinct trade-off between \textit{internalization} (baking capability into weights) and \textit{deliberation} (spending compute at inference).

\paragraph{Policy Optimization via RL.}
The \textbf{SimpleRL-Zoo (GRPO)} baseline demonstrates the efficacy of reinforcement learning in aligning the policy distribution. 
Starting from the base model's Pass@1 of 57.7\%, GRPO significantly lifts the zero-shot generation performance to \textbf{74.7\%}. 
This confirms that coupled optimization effectively concentrates probability mass on correct reasoning traces.

\paragraph{Guided Search via Verification.}
In contrast, \textbf{ARTS} operates with the frozen base model, which retains the original entropy but lacks zero-shot precision. 
However, by applying Best-of-$N$ search ($N=16$) guided by our 1.5B verifier, ARTS improves the performance from 57.7\% to \textbf{74.6\%}.
Notably, this result is statistically indistinguishable from the fully fine-tuned GRPO baseline (74.7\%).
This finding suggests that \textbf{inference-time search} can serve as a viable alternative to expensive full-parameter fine-tuning, achieving comparable downstream performance by unlocking the latent potential of the pre-trained model.

\paragraph{Gap Closing Efficiency.}
To contextualize the verifier's performance, we analyze the \textit{Gap Closed} metric.
The theoretical upper bound (Oracle@16) for the base model is 88.6\%.
While standard Pointwise PRM and Pairwise RM baselines close 52.1\% and 31.3\% of this gap respectively, ARTS closes \textbf{54.7\%}.
This indicates that the Flow Matching objective is highly effective at identifying correct reasoning paths within the noisy search space, offering a robust signal for guiding search processes without requiring policy updates.

\subsection{Robustness on High-Entropy Tasks}
\begin{table}[htbp]
\centering
\caption{\textbf{Domain-Specific Analysis (BoN@16).} 
Pointwise PRM (MSE) shows strength in standard procedural tasks (Algebra). 
However, ARTS (Flow Matching) demonstrates higher robustness in high-entropy domains (Counting \& Probability) and on the hardest problem subset (Level 5).}
\label{tab:domain_breakdown}

\small   

\begin{tabular}{lccc}
\toprule
\textbf{Category} & \textbf{Pointwise PRM} & \textbf{ARTS (RFM)} & \textbf{$\Delta$} \\
\midrule
Algebra & \textbf{89.5\%} & 88.7\% & -0.8\% \\
Number Theory & \textbf{87.1\%} & 80.7\% & -6.4\% \\
Geometry & 58.5\% & \textbf{61.0\%} & +2.5\% \\
Counting \& Prob. & 73.7\% & \textbf{79.0\%} & \textbf{+5.3\%} \\
Level 5 (Hardest) & 47.8\% & \textbf{50.0\%} & +2.2\% \\
\bottomrule
\end{tabular}
\end{table}

While the aggregate performance of ARTS and the Pointwise PRM baseline is comparable, breaking down the results by domain reveals distinct inductive biases in their training objectives.
Table~\ref{tab:domain_breakdown} details the performance across mathematical sub-domains.

\paragraph{Procedural vs. Combinatorial Reasoning.}
We observe a performance trade-off based on the nature of the reasoning task:
\begin{itemize}
    \item \textbf{Linear Reasoning:} In domains like \textit{Algebra} and \textit{Number Theory}, the reasoning process is often linear and procedural. Here, the Pointwise PRM performs slightly better. This suggests that when independent steps strongly correlate with the final outcome, dense step-level supervision is highly effective.
    
    \item \textbf{Branching Reasoning:} In contrast, \textit{Counting \& Probability} problems often involve complex decision trees where a local error can propagate invisibly. In this high-entropy domain, ARTS achieves a notable gain of \textbf{+5.3\%}. 
    This indicates that the \textbf{Flow Matching} objective, by enforcing global consistency across the search tree ($\log F(s) \approx \log \sum F(s')$), may be structurally better suited for navigating combinatorial spaces where the value of a step depends heavily on the aggregation of its future paths.
\end{itemize}

\paragraph{Robustness on Hard Problems.}
On the \textbf{Level 5} subset, which represents the long-tail of problem difficulty, ARTS maintains a \textbf{+2.2\%} advantage over the PRM baseline. 
This aligns with our theoretical motivation: hard problems typically require exploring rare reasoning paths. 
By modeling the probability mass of the entire subtree rather than just the immediate step quality, the flow-based verifier appears to offer more robust guidance for recovering these elusive solutions.

\subsection{Analysis on Long-Tail Reasoning Distributions}
\label{sec:resurrection}

Standard metrics often obscure model performance on the "long tail" of reasoning distributions. 
To rigorously test the limits of ARTS, we construct two stress-test subsets where correct reasoning paths are statistically invisible.

\paragraph{Scenario A: The "Needle in a Haystack".}
We select problems where the Base Proposer generates \textbf{at most 2} correct solutions out of 16 samples (Probability $\le 12.5\%$), simulating the "Rare Traces" defined in Section~\ref{sec:theoretical_analysis}.

\paragraph{Scenario B: The "Extinction Zone".}
We isolate problems where the \textbf{GRPO Baseline} yields 0\% accuracy, representing reasoning patterns that were effectively "extinguished" by the mode-seeking optimization.

\begin{figure}[t]
    \label{}
    \centering
    \includegraphics[width=\linewidth]{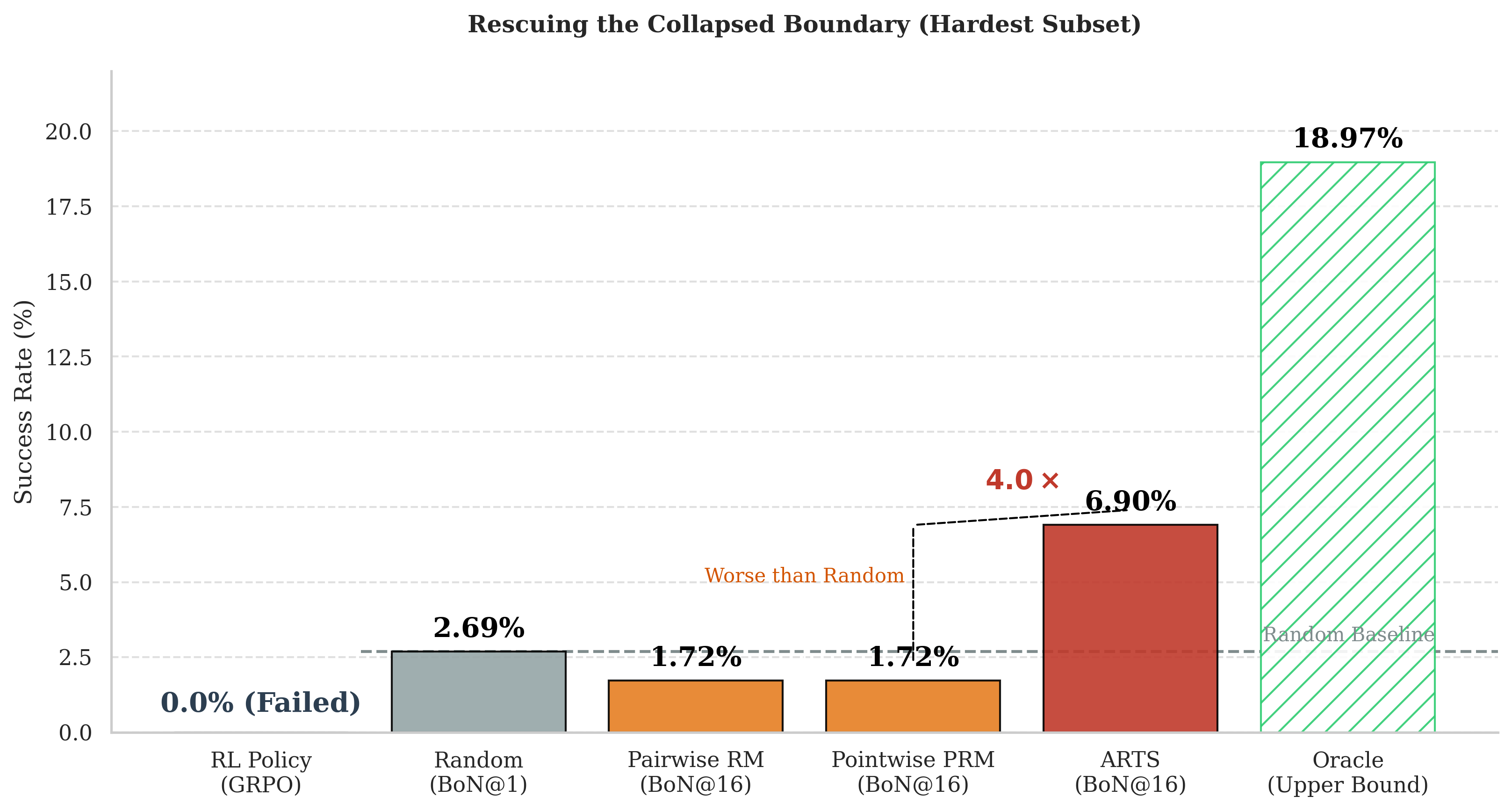} 
    \caption{
    \textbf{Performance on the ``Extinction Set''.} 
    We compare recovery rates on the hardest subset where the RL policy (GRPO) fails completely (0\%).
    While standard verifiers (PRM/RM) suppress rare traces below the random sampling baseline (Dashed Line), ARTS uniquely achieves positive recovery.
    This confirms that Flow Matching preserves valid reasoning signals in the long tail where standard discriminative objectives collapse.
    }
    \label{fig:rescuing_boundary}
\end{figure}

\begin{table}[t]
\centering
\caption{\textbf{Stress Test on Long-Tail Reasoning.} 
\textbf{Panel A:} Problems with scarce solutions ($\le 2/16$). 
\textbf{Panel B:} Problems where RL (GRPO) fails completely. 
ARTS demonstrates superior sensitivity in both regimes compared to random sampling and standard verifiers.}
\label{tab:stress_test}

\footnotesize   

\setlength{\tabcolsep}{4pt}  

\begin{tabular}{lccc}
\toprule
\multicolumn{4}{c}{\textbf{Panel A: Rare Solutions (Base Prob $\le$ 12.5\%)}} \\
\midrule
\textbf{Category} & \textbf{Random} & \textbf{ARTS (BoN@16)} & \textbf{Gain} \\
\midrule
Global Average & 8.5\% & \textbf{27.0\%} & $\mathbf{3.2\times}$ \\
\textit{Counting \& Prob.} & 7.5\% & \textbf{60.0\%} & $\mathbf{8.0\times}$ \\
\midrule
\multicolumn{4}{c}{\textbf{Panel B: Extinguished Solutions (GRPO Pass@16 = 0\%)}} \\
\midrule
\textbf{Metric} & \textbf{Random} & \textbf{PRM / RM} & \textbf{ARTS (Ours)} \\
\midrule
BoN@16 Accuracy & 2.7\% & 1.7\% (Failure) & \textbf{6.9\%} \\
Gap Closed & 0.0\% & -6.0\% (Suppression) & \textbf{+25.8\%} \\
\bottomrule
\end{tabular}
\end{table}

\paragraph{Analysis: The Flow Advantage.}
The results in Figure~\ref{fig:rescuing_boundary} and Table~\ref{tab:stress_test} reveal the unique robustness of the Flow Matching objective:

\begin{itemize}
    \item \textbf{Amplification of Rare Signals (Panel A):} 
    Even when the signal-to-noise ratio is extremely low ($1:15$), ARTS identifies correct paths with high precision. In combinatorial domains like \textit{Counting \& Probability}, it amplifies the success rate by an order of magnitude ($8\times$), demonstrating its capability to navigate sparse search spaces.
    
    \item \textbf{Resilience to Suppression (Panel B):} 
    Crucially, on the "Extinction Set" where RL fails, standard verifiers (both Pairwise RM and Pointwise PRM) collapse, achieving only $\sim 1.7\%$ accuracy. This is notably \textbf{worse than random sampling} ($2.7\%$), suggesting that discriminative objectives trained on imbalanced data actively penalize valid but rare logic as "out-of-distribution" noise.
    In contrast, ARTS is the \textbf{only method} to achieve positive recovery ($6.9\%$), outperforming standard verifiers by \textbf{$4.0\times$}. 
    By modeling global probability flow rather than local step preference, ARTS preserves the signal of rare correct paths, effectively bridging the gap (+25.8\%) that coupled optimization leaves behind.
\end{itemize}

\section{Discussion and Conclusion}
\label{sec:discussion}

In this work, we investigated the optimization dynamics of Reinforcement Learning with Verifiable Rewards (RLVR). 
Our theoretical analysis identifies a structural vulnerability in standard coupled optimization: the \textit{extinction} of valid but rare reasoning traces caused by the mechanistic interplay between low-temperature sampling and softmax normalization.
To mitigate this \textbf{"High-Pass Filter"} effect, we proposed \textbf{Amortized Reasoning Tree Search (ARTS)}, a framework that fundamentally decouples the search space generator (Proposer) from the value estimator (Verifier) using a flow-matching objective.

\subsection{The Trade-off: Internalization vs. Deliberation}

Our findings illuminate a fundamental trade-off in the design of reasoning systems, balancing \textit{parameter updates} against \textit{inference compute}:

\begin{itemize}
    \item \textbf{Internalization (System 1):} Coupled methods like GRPO aim to \textbf{internalize} successful reasoning patterns directly into the policy's weights. This "compiles" search into intuition, yielding high zero-shot efficiency. However, as shown in our extinction analysis, this process can aggressively prune the long-tail diversity required for hard problems.
    
    \item \textbf{Deliberation (System 2):} In contrast, ARTS prioritizes \textbf{deliberation}. By maintaining a high-entropy generator and relying on a flow-calibrated verifier to guide inference-time search, it preserves the model's latent boundaries. 
    While this shifts the computational burden from training to inference, recent scaling laws for inference compute~\cite{brown2024large} suggest that this may be a more scalable path for solving complex reasoning tasks where a single "best" path is insufficient.
\end{itemize}

We view these approaches not as mutually exclusive, but as complementary strategies on the \textbf{Pareto frontier} of intelligence: RLVR for efficiency in common scenarios, and Flow-Guided Search for robustness in the long tail.

\subsection{Limitations}

We acknowledge several limitations in our current study that merit further investigation:
\begin{itemize}
    \item \textbf{Bounded by the Proposer:} ARTS operates as a selector, not a creator. It cannot generate solutions that are structurally absent from the base model's support. If the frozen proposer fails to cover the correct reasoning path within the sampling budget ($N$), the verifier is powerless to recover it.
    \item \textbf{Inference Latency:} Unlike end-to-end RL policies which can be deployed greedily, ARTS relies on Best-of-$N$ sampling or tree search. This incurs higher latency and cost, making it less suitable for latency-sensitive applications.
    \item \textbf{Heuristic Reward Shaping:} Our Flow Matching objective currently relies on a manually designed soft reward scheme ($10$ vs $0.1$) to ensure numerical stability. While empirically effective, deriving these values from a principled probabilistic framework remains an open question.
\end{itemize}

\subsection{Future Directions}

The efficacy of ARTS suggests several promising avenues for future research:
\begin{itemize}
    \item \textbf{Iterative ARTS (Expert Iteration):} The high-quality trajectories filtered by ARTS could serve as synthetic data to fine-tune the base model, akin to \textbf{Expert Iteration (ExIt)}~\cite{anthony2017thinking}. This closed-loop approach could progressively improve the proposer's efficiency without suffering from the immediate mode collapse observed in pure RL.
    \item \textbf{Beyond Mathematics:} While we focused on math, the Flow Matching objective is theoretically applicable to any domain with sparse, tree-structured reasoning, such as code generation and agentic planning.
    \item \textbf{Unified "Anytime" Architectures:} An intriguing direction is to unify the policy and verifier into a single model that can dynamically switch between fast generation (System 1) and slow search (System 2) based on the estimated difficulty of the problem.
\end{itemize}

In conclusion, we encourage the community to re-examine the role of search in the era of large language models. 
By treating reasoning not just as token prediction but as a \textbf{flow matching problem over a sparse reasoning space}, we can unlock latent capabilities that are currently hidden behind the visibility barrier of standard optimization.

\bibliography{iclr2025_conference}
\bibliographystyle{iclr2025_conference}

\appendix
\section{Proposition of Theorem \ref{prop:decay_law}}
\label{app:proof_thm1}

In this section, we provide the formal derivation of the Probability Decay Law for unsampled reasoning traces.

\subsection{Setup and Definitions}

Let the policy at iteration $k$ be parameterized by a scoring function (logits) $f_k: \mathcal{T} \to \mathbb{R}$, such that the probability of a trace $\tau$ is given by the Softmax distribution:
\begin{equation}
    \pi_k(\tau) = \frac{\exp(f_k(\tau))}{Z_k}, \quad \text{where} \quad Z_k = \sum_{\tau' \in \mathcal{T}} \exp(f_k(\tau')).
\end{equation}

In standard RLVR (e.g., GRPO/PPO), the optimization objective is approximated using a finite batch of samples $\mathcal{S}_k$. The gradient update is sparse: it only affects the logits of traces present in $\mathcal{S}_k$. 
The update rule for the logits can be written as:
\begin{equation}
    f_{k+1}(\tau) = 
    \begin{cases} 
    f_k(\tau) + \frac{R(\tau)}{\beta} & \text{if } \tau \in \mathcal{S}_k \quad (\text{Reinforced}) \\
    f_k(\tau) & \text{if } \tau \notin \mathcal{S}_k \quad (\text{Frozen})
    \end{cases}
\end{equation}
where $R(\tau)$ is the reward and $\beta$ is the KL coefficient (temperature of the update).

\subsection{Derivation of Partition Function Shift}

We first analyze how the normalization constant (partition function) $Z_{k+1}$ changes relative to $Z_k$.
By definition:
\begin{equation}
    Z_{k+1} = \sum_{\tau \in \mathcal{T}} \exp(f_{k+1}(\tau)).
\end{equation}
We split the summation over the reasoning space $\mathcal{T}$ into two disjoint sets: the sampled set $\mathcal{S}_k$ and the unsampled set $\mathcal{T} \setminus \mathcal{S}_k$.
\begin{align}
    Z_{k+1} &= \sum_{\tau \in \mathcal{S}_k} \exp\left(f_k(\tau) + \frac{R(\tau)}{\beta}\right) + \sum_{\tau \notin \mathcal{S}_k} \exp(f_k(\tau)) \\
    &= \sum_{\tau \in \mathcal{S}_k} \exp(f_k(\tau)) \cdot \exp\left(\frac{R(\tau)}{\beta}\right) + \sum_{\tau \notin \mathcal{S}_k} \exp(f_k(\tau)).
\end{align}
Recall that for the original distribution, $Z_k = \sum_{\tau \in \mathcal{S}_k} \exp(f_k(\tau)) + \sum_{\tau \notin \mathcal{S}_k} \exp(f_k(\tau))$.
Thus, the sum over unsampled traces can be expressed as:
\begin{equation}
    \sum_{\tau \notin \mathcal{S}_k} \exp(f_k(\tau)) = Z_k - \sum_{\tau \in \mathcal{S}_k} \exp(f_k(\tau)).
\end{equation}
Substituting this back into the expression for $Z_{k+1}$:
\begin{align}
    Z_{k+1} &= \sum_{\tau \in \mathcal{S}_k} \exp(f_k(\tau)) \cdot \exp\left(\frac{R(\tau)}{\beta}\right) + Z_k - \sum_{\tau \in \mathcal{S}_k} \exp(f_k(\tau)) \\
    &= Z_k + \sum_{\tau \in \mathcal{S}_k} \exp(f_k(\tau)) \left[ \exp\left(\frac{R(\tau)}{\beta}\right) - 1 \right].
\end{align}
Dividing both sides by $Z_k$, and noting that $\frac{\exp(f_k(\tau))}{Z_k} = \pi_k(\tau)$, we obtain the ratio of partition functions:
\begin{equation}
\label{eq:partition_ratio}
    \frac{Z_{k+1}}{Z_k} = 1 + \sum_{\tau \in \mathcal{S}_k} \pi_k(\tau) \left[ \exp\left(\frac{R(\tau)}{\beta}\right) - 1 \right].
\end{equation}
We define the \textbf{Mass Amplification Rate} $\alpha_k$ as:
\begin{equation}
    \alpha_k \triangleq \sum_{\tau \in \mathcal{S}_k} \pi_k(\tau) \left[ \exp\left(\frac{R(\tau)}{\beta}\right) - 1 \right].
\end{equation}
Thus, $Z_{k+1} = Z_k (1 + \alpha_k)$.

\subsection{Probability Update for Unsampled Traces}

Now consider a specific reasoning trace $\tau_{\text{rare}}$ that was \textbf{not} sampled ($\tau_{\text{rare}} \notin \mathcal{S}_k$). Its logit remains unchanged: $f_{k+1}(\tau_{\text{rare}}) = f_k(\tau_{\text{rare}})$.
The probability at step $k+1$ is:
\begin{align}
    \pi_{k+1}(\tau_{\text{rare}}) &= \frac{\exp(f_{k+1}(\tau_{\text{rare}}))}{Z_{k+1}} \\
    &= \frac{\exp(f_k(\tau_{\text{rare}}))}{Z_k (1 + \alpha_k)} \\
    &= \frac{\exp(f_k(\tau_{\text{rare}}))}{Z_k} \cdot \frac{1}{1 + \alpha_k} \\
    &= \pi_k(\tau_{\text{rare}}) \cdot \frac{1}{1 + \alpha_k}.
\end{align}
Assuming positive rewards $R(\tau) > 0$ (which is true for correct reasoning traces), we have $\exp(R/\beta) > 1$, implying $\alpha_k > 0$.
Therefore, the decay factor $\gamma_k = \frac{1}{1+\alpha_k} < 1$.
\qed

\section{Theoretical Analysis of Flow-Guided Search}
\label{app:inference_optimality}

In this section, we provide the theoretical justification for using the learned flow function $F_\theta$ to guide the Beam Search process during inference. We show that minimizing the RFM objective implies that $F_\theta$ recovers the correct ranking of candidates within the support of the Proposer.

\subsection{Tree-Conditional Optimality of RFM}

Recall the Reasoning Flow Matching (RFM) objective on a fixed training graph (tree) $\mathcal{G}$:
\begin{equation}
\mathcal{L}_{\mathrm{RFM}}(\theta) = \sum_{s \in \mathcal{G}} \left( \log F_{\theta}(s) - \log \sum_{s' \in \mathcal{C}_{obs}(s)} F_{\theta}(s') \right)^2.
\end{equation}

Assuming the model has sufficient capacity (realizability), the global minimum $\mathcal{L}_{\mathrm{RFM}}=0$ implies the flow conservation holds exactly on the graph:
\begin{equation}
\label{eq:flow_conservation}
F_{\theta^*}(s) = \sum_{s' \in \mathcal{C}_{obs}(s)} F_{\theta^*}(s'), \quad \forall s \in \mathcal{G}.
\end{equation}
By unrolling this recursion to the terminal states $\mathcal{S}_T$, we obtain:
\begin{equation}
F_{\theta^*}(s) = \sum_{\tau \in \mathcal{T}(s)} R(\tau),
\end{equation}
where $\mathcal{T}(s)$ is the set of terminal trajectories in the graph reachable from $s$.

This result establishes \textbf{Tree-Conditional Optimality}: for a fixed search tree generated by the Proposer (whether during training or inference), the optimal flow value $F_{\theta^*}(s)$ represents the \textbf{total unnormalized probability mass of correct solutions} residing in the subtree rooted at $s$.

\subsection{Optimality of Beam Selection}

Consider the inference step at state $s_t$, where the Proposer $\pi_0$ generates a candidate set $\mathcal{A}_K = \{a^{(1)}, \dots, a^{(K)}\}$. This induces $K$ child states $\{s^{(1)}_{t+1}, \dots, s^{(K)}_{t+1}\}$.
Our Beam Search strategy selects the top-$B$ candidates maximizing $F_\theta(s^{(k)}_{t+1})$.

Under the Soft Reward scheme ($R=10$ for correct, $R=0.1$ for incorrect):
\begin{itemize}
    \item If a candidate $s^{(i)}$ leads to at least one discoverable correct path, $F(s^{(i)}) \ge 10$.
    \item If a candidate $s^{(j)}$ leads only to incorrect paths (within the Proposer's support), $F(s^{(j)}) = \sum 0.1 \ll 10$.
\end{itemize}
Therefore, ranking by $F_\theta$ is equivalent to maximizing the probability of keeping a valid solution in the beam. Specifically,
\begin{equation}
\arg\max_{a \in \mathcal{A}_K} F_\theta(s_{next}) \iff \arg\max_{a \in \mathcal{A}_K} \mathbb{E}_{\pi_0}[\text{Future Correctness} \mid s, a].
\end{equation}
This confirms that ARTS inference structurally approximates the optimal search policy bounded by the Proposer's exploration capability.

\section{Derivation of Visibility Amplification}
\label{app:disc_amp}

This appendix provides the derivation for the ``Visibility Booster'' factor $K^H$ referenced in Section~\ref{sec:ARTS}, demonstrating how beam search with proposals counteracts the base model's likelihood decay.

\subsection{Proposer Inclusion Probability}
Let $\pi_0(a \mid s)$ be the base model's probability of selecting action $a$. During ARTS inference (and Sparse Deep Fork training), we sample a candidate set $\mathcal{A}_K$ of $K$ actions i.i.d. from $\pi_0$.
The probability that a specific target action $a^*$ is included in $\mathcal{A}_K$ is:
\begin{equation}
    P(a^* \in \mathcal{A}_K) = 1 - (1 - \pi_0(a^* \mid s))^K.
\end{equation}

\subsection{Rare-Action Regime}
In complex reasoning tasks, correct actions often lie in the long tail, satisfying $\pi_0(a^* \mid s) \ll 1/K$.
Applying the first-order Taylor expansion $(1-x)^K \approx 1 - Kx$, we approximate the inclusion probability as:
\begin{equation}
    P(a^* \in \mathcal{A}_K) \approx K \cdot \pi_0(a^* \mid s).
\end{equation}
This indicates that the ``effective prior'' of the action is linearly amplified by the proposal width $K$.

\subsection{Trajectory Discovery}
A complete reasoning trajectory $\tau = (a_1, \dots, a_H)$ is discovered if and only if each action $a_t$ is included in the respective proposal set.
Assuming conditional independence between steps, the total discovery probability scales as:
\begin{equation}
    P_{\mathrm{disc}}(\tau) \approx \prod_{t=1}^H \left( K \cdot \pi_0(a_t \mid s_{t-1}) \right) = K^H \cdot \pi_0(\tau).
\end{equation}
For a reasoning depth $H=10$ and $K=4$, the factor $K^H \approx 10^6$. 
This amplification explains why ARTS can recover and verify reasoning paths that are statistically impossible to sample using standard temperature-based decoding (where probability would be just $\pi_0(\tau)$).

\section{Detailed Experimental Results}
\label{app:detailed_results}

In this section, we provide the comprehensive performance breakdown of ARTS and baselines across different problem difficulties and mathematical sub-domains. 
All results reported here are \textbf{Best-of-N (BoN@16)} accuracies using the frozen \texttt{Qwen-2.5-7B} Proposer and \texttt{Qwen-2.5-1.5B} Verifiers.

\subsection{Performance by Difficulty Level}

Table~\ref{tab:full_level_breakdown} details the performance scaling across the five difficulty levels defined in the MATH benchmark.
While standard Pairwise Ranking (RM) and Pointwise Regression (PRM) perform well on easier problems (Levels 1--3), \textbf{ARTS} demonstrates superior scaling on harder problems.
Notably, on \textbf{Level 5} (the hardest subset), ARTS achieves \textbf{50.0\%} accuracy, surpassing the PRM baseline (47.8\%) and the RM baseline (46.3\%). This confirms that the Flow Matching objective is particularly effective for recovering rare reasoning traces that are statistically unlikely in the base distribution.

\begin{table*}[t]
\centering
\caption{\textbf{Detailed Breakdown by Difficulty Level (BoN@16).} 
Base represents random sampling ($N=1$). Verifiers re-rank $N=16$ samples. 
ARTS shows stronger performance on higher-difficulty problems (Levels 4--5).}
\label{tab:full_level_breakdown}

\small
\setlength{\tabcolsep}{5pt}

\begin{tabular}{lccccc}
\toprule
\textbf{Difficulty} & \textbf{Base} & \textbf{Pairwise RM} & \textbf{Pointwise PRM} & \textbf{ARTS (RFM)} & \textbf{Best} \\
\midrule
Level 1 (Easiest) & 81.7\% & \textbf{93.0\%} & \textbf{93.0\%} & 88.4\% & PRM/RM \\
Level 2 & 71.1\% & 82.2\% & \textbf{90.0\%} & 87.8\% & PRM \\
Level 3 & 72.0\% & 75.2\% & \textbf{86.7\%} & 85.7\% & PRM \\
Level 4 & 55.3\% & 64.1\% & 72.7\% & \textbf{77.3\%} & \textbf{ARTS} \\
Level 5 (Hardest) & 32.2\% & 46.3\% & 47.8\% & \textbf{50.0\%} & \textbf{ARTS} \\
\midrule
\textbf{Global Avg} & 57.7\% & 67.4\% & 73.8\% & \textbf{74.6\%} & \textbf{ARTS} \\
\bottomrule
\end{tabular}
\end{table*}

\subsection{Performance by Problem Category}

Table~\ref{tab:full_category_breakdown} breaks down performance by mathematical subject. 
This comparison reveals the distinct inductive biases of the training objectives:
\begin{itemize}
    \item \textbf{Pointwise PRM (MSE)} excels in \textbf{Algebra} (+0.8\% over ARTS) and \textbf{Number Theory} (+6.4\% over ARTS). These domains typically involve linear, procedural derivation where independent step-level scoring is sufficient.
    \item \textbf{ARTS (RFM)} dominates in \textbf{Counting \& Probability} (+5.3\% over PRM), \textbf{Geometry} (+2.5\% over PRM), and \textbf{Precalculus} (+3.6\% over PRM). These domains involve high combinatorial complexity, branching logic, or spatial reasoning, where the global consistency enforced by Flow Matching provides a stronger signal than local value estimation.
\end{itemize}

\begin{table*}[t]
\centering
\caption{\textbf{Detailed Breakdown by Problem Category (BoN@16).} 
Comparison of different verifier objectives across mathematical sub-domains. 
ARTS shows consistent gains in high-entropy domains (Counting, Geometry, Precalculus), 
while Pointwise PRM performs strongly in procedural domains (Algebra, Number Theory).}
\label{tab:full_category_breakdown}

\small
\setlength{\tabcolsep}{5pt}

\begin{tabular}{lccccc}
\toprule
\textbf{Category} & \textbf{Base} & \textbf{Pairwise RM} & \textbf{Pointwise PRM} & \textbf{ARTS (RFM)} & \textbf{Gap} \\
\midrule
Algebra & 72.9\% & 81.5\% & \textbf{89.5\%} & 88.7\% & -0.8\% \\
Number Theory & 60.2\% & 64.5\% & \textbf{87.1\%} & 80.7\% & -6.4\% \\
Prealgebra & 69.3\% & 76.8\% & 81.7\% & \textbf{81.7\%} & 0.0\% \\
Intermediate Algebra & 42.9\% & 56.7\% & 53.6\% & \textbf{57.7\%} & \textbf{+4.1\%} \\
\midrule
Geometry & 49.5\% & 56.1\% & 58.5\% & \textbf{61.0\%} & \textbf{+2.5\%} \\
Precalculus & 40.9\% & 53.6\% & 58.9\% & \textbf{62.5\%} & \textbf{+3.6\%} \\
Counting \& Prob. & 50.8\% & 65.8\% & 73.7\% & \textbf{79.0\%} & \textbf{+5.3\%} \\
\bottomrule
\end{tabular}
\end{table*}

\section{Implementation Details}
\label{app:implementation_details}

\subsection{Prompt Templates}
To ensure reproducibility, we provide the exact prompt templates used for both the Proposer (generation) and the Verifier (flow estimation).

\paragraph{Proposer Prompts.}
For the generative policy (Qwen-2.5-7B), we utilize the standard ChatML format. 
Although our Proposer utilizes a base model checkpoint rather than an instruction-tuned one, we adopt this specific template following \textbf{SimpleRL-Zoo}~\cite{zeng2025simplerl}, as their empirical results confirm that this structure effectively elicits mathematical Chain-of-Thought (CoT) from Qwen base models and facilitates robust answer extraction via the boxed format.
The template is structured as follows:

\begin{lstlisting}[
    basicstyle=\ttfamily\footnotesize, 
    breaklines=true, 
    frame=single, 
    columns=fullflexible,
    keepspaces=true
]
<|im_start|>system
You are a helpful assistant.<|im_end|>
<|im_start|>user
{Question}
Please reason step by step, and put your final answer within \boxed{}.<|im_end|>
<|im_start|>assistant
\end{lstlisting}

\paragraph{Verifier Prompts.}
For the ARTS verifier, we employ a lightweight instruction-tuned model (Qwen-2.5-1.5B-Instruct). We adopt a simplified prompt structure to reduce context overhead, where reasoning steps are concatenated with newline characters:

\begin{lstlisting}[
    basicstyle=\ttfamily\footnotesize, 
    breaklines=true, 
    frame=single, 
    columns=fullflexible,
    keepspaces=true
]
Question: {Question}
Answer: Let's think step by step.
{Step 1}
{Step 2}
...
{Current Step}
\end{lstlisting}

\subsection{Generation Hyperparameters}
We generate the \textbf{Sparse Deep Fork} training data and perform inference using the \textbf{Qwen-2.5-7B} base model.
For sampling, we set the temperature to $T=0.6$ and nucleus sampling to top-$p=0.95$.
This specific temperature setting is selected following recent empirical findings in RLVR~\cite{zeng2025simplerl, yue2025does}, which demonstrate that $T=0.6$ strikes an optimal trade-off for Qwen-2.5 models: it maintains high generation accuracy while providing sufficient diversity to cover valid reasoning paths in the long tail.

\end{document}